\documentclass[11pt]{article}

\usepackage[
	persons={Nima, June},
	authors=blocks,
	bibliosources={refs.bib}
]{pomegranate}

\usepackage[utf8]{inputenc}
\usepackage{enumitem}

\DeclareOperator{\diag}
\DeclareOperator{\weight}
\DeclareOperator{\newt}

\DeclareDocumentTextCommand{\SharpP}{}{\Class{\#P}}
\DeclareDocumentMathCommand{\tmix}{}{t_{\operatorname{mix}}}
\DeclareDocumentMathCommand{\dtv}{}{d_{\operatorname{TV}}}
\DeclareDocumentMathCommand{\ef}{m m}{{#1\star #2}}
\DeclareDocumentMathCommand{\I}{}{\mathcal{I}}
\DeclareDocumentMathCommand{\B}{}{\mathcal{B}}
\DeclareDocumentMathCommand{\cE}{}{\mathcal{E}}
\DeclareDocumentMathCommand{\cN}{}{\mathcal{N}}

\DeclareOperator{\OPT}
\DeclareOperator{\LS}

\title{From Sampling to Optimization on Discrete Domains with Applications to Determinant Maximization}
\author[1]{Nima Anari}
\author[1]{Thuy-Duong Vuong}
\affil[1]{\small Stanford University, \textsf{anari@cs.stanford.edu}, \textsf{tdvuong@stanford.edu}}
\date{}

\begin{document}
    \maketitle
    \begin{abstract}
        We show a connection between sampling and optimization on discrete domains. For a family of distributions $\mu$ defined on size $k$ subsets of a ground set of elements that is closed under external fields, we show that rapid mixing of natural local random walks implies the existence of simple approximation algorithms to find $\max \mu(\cdot)$. More precisely we show that if (multi-step) down-up random walks have spectral gap at least inverse polynomially large in $k$, then (multi-step) local search can find $\max \mu(\cdot)$ within a factor of $k^{O(k)}$. As the main application of our result, we show a simple nearly-optimal $k^{O(k)}$-factor approximation algorithm for MAP inference on nonsymmetric DPPs. This is the first nontrivial multiplicative approximation for finding the largest size $k$ principal minor of a square (not-necessarily-symmetric) matrix $L$ with $L+L^\intercal\succeq 0$.

        We establish the connection between sampling and optimization by showing that an exchange inequality, a concept rooted in discrete convex analysis, can be derived from fast mixing of local random walks. We further connect exchange inequalities with composable core-sets for optimization, generalizing recent results on composable core-sets for DPP maximization to arbitrary distributions that satisfy either the strongly Rayleigh property or that have a log-concave generating polynomial.
    \end{abstract}
    
    \section{Introduction}
    Sampling and optimization are fundamental tasks in mathematics, statistical physics, and various subfields of computer science such as cryptography, differential privacy, machine learning, and artificial intelligence. In continuous settings, sampling and optimization are known to be intimately connected; convex sets, and more generally log-concave distributions, are the natural domains where either task is algorithmically tractable. For a more formal treatment of this connection in continuous settings see \cite{LV06, LSV18}.

    On discrete/combinatorial domains, the relationship between sampling and optimization is less clear. For example, the intersection of two matroids is easy to optimize over, but not known to be easy to sample from, and the opposite holds for determinantal point processes, which are easy to sample from \cite[see, e.g.,][]{AOR16} and hard to optimize \cite{CM10}.

    The goal of this work is to establish a new connection between sampling and optimization in discrete settings. For a family of distributions $\mu$ defined on size $k$ subsets of a ground set of elements\footnote{The restriction of the domain to size $k$ subsets of a ground set should be thought of as a ``canonical form'';  many other discrete domains can be naturally transformed into this form.} that is closed under external fields, we show that rapid mixing of natural local random walks implies the existence of simple approximation algorithms to find $\max \mu(\cdot)$. More specifically, we show that local search can approximately find $\max \mu(\cdot)$ within a nearly-optimal approximation factor.

    We study a family of natural local search algorithms (\cref{alg:localsearch}) to find $\max \mu(\cdot)$. These algorithms start with a set $S$, and repeatedly try to increase $\mu(S)$ by swapping a constant number of elements in $S$ with elements outside of $S$ until no more improvements can be made.
    
    More formally, suppose that the domain of the objective $\mu$ is the collection of size $k$ subsets of the ground set $[n]=\set{1,\dots,n}$, which we denote by $\binom{[n]}{k}$. Then, local search is defined with a parameter $r \geq 0$ which specifies the ``local neighborhood'' the algorithm searches over in each iteration. The $r$-neighborhood of $S\in \binom{[n]}{k}$ are all the sets that can be reached by swapping at most $r$ elements:
    	\[ \cN_r(S) := \set*{T \in \binom{[n]}{k} \given \card{S-T}\leq r }. \]
    Each iteration of local search goes from a set $S$ to $\hat{S}\in \cN_r(S)$ which maximizes $\mu(\hat{S})$. If we reach a local optimum, i.e., $S=\hat{S}$, then $\mu(S)=\max\set{\mu(T)\given T\in \cN_r(S)}$. 
    
    We show in this work that rapid mixing of natural local random walks, the \emph{(multi-step) down-up random walks}, designed to sample from $\mu$ and related distributions, implies that \emph{local maxima} of $\mu$ are \emph{approximate global maxima}.
    \begin{definition}[Down-Up Random Walks]\label{def:local-walk}
	For a density $\mu:\binom{[n]}{k}\to\R_{\geq 0}$, and an integer $\l\leq k$, we define the $k\leftrightarrow\l$ down-up random walk as the sequence of random sets $S_0, S_1,\dots$ generated by the following algorithm:
	\begin{Algorithm*}
		\For{$t=0,1,\dots$}{
			Select $T_t$ uniformly at random from subsets of size $\l$ of $S_t$.\;
			Select $S_{t+1}$ with probability $\propto \mu(S_{t+1})$ from supersets of size $k$ of $T_t$.\;
		}
	\end{Algorithm*}
    \end{definition}

    This random walk is time-reversible, always has $\mu$ as its stationary distribution, and moreover has positive real eigenvalues \cite[see, e.g.,][]{ALO20}. This random walk, specially for the case of $\l=k-1$, has received a lot of attention in the literature on high-dimensional expanders \cite[see, e.g.,][]{LLP17,KO18,DK17,KM16,AL20,ALO20}. Each step of this random walk can be efficiently implemented as long as $k-\l=O(1)$ and we have oracle access to $\mu$. We remark that down-up walks generalize other well-known local random walks like the Glauber dynamics \cite[see, e.g.,][]{ALO20}. Note that the down-up random walk is \emph{local} in the sense that $S_{t+1}\in \cN_{k-\l}(S_t)$. Naturally, we tie mixing of these random walks to local search with $r=k-\l$ neighborhoods.

    There has been a recent surge of interest in analyzing the mixing properties of down-up random walks due to a number of breakthrough applications to open problems in sampling and counting \cite{ALOV19,AL20,ALO20,alimohammadi2021fractionally,CLV21a,CLV21b,FGYZ21,Liu21,JPV21,BCCPSV21,CFYZ21,ALOVV21,AJKPV21}.
    Key to many of these works was the notion of spectral independence. \textcite{alimohammadi2021fractionally} introduced a stronger notion called fractional log-concavity, and showed that it implies a $k^{-O(1)}$ lower bound on the spectral gap of $k \leftrightarrow (k-O(1))$-down-up random walks on $\mu.$ We remark that fractional log-concavity, unlike spectral independence, is preserved under external fields, formally defined as follows.

    For a distribution $\mu$ on $\binom{[n]}{k}$ and $\lambda = (\lambda_1,\dots, \lambda_n) \in \R^{n}_{>0}$, the \emph{$\lambda$-external field} applied to $\mu$ is another distribution on $\binom{[n]}{k}$, denoted by $\lambda \ast \mu$, defined up to normalization as follows:
    \[\mathbb{P}_{\lambda \ast \mu}[S] \propto \mu(S)\cdot \prod_{i \in S}\lambda_i.\]

As established in \cite{alimohammadi2021fractionally}, various distributions of interest involving determinants are fractionally log-concave. For  a fractionally log-concave distribution $\mu$, the $k\leftrightarrow (k-O(1))$-down-up walk on $\lambda\ast \mu$ has inverse-polynomially large spectral gap, even when an arbitrary external field $\lambda\in \R_{\geq 0}^n$ is applied to $\mu$. We show that this property\footnote{Curiously, in continuous settings applying an external field also preserves log-concavity, the standard of algorithmic tractability for sampling; applying an external field is the same as multiplication by a log-linear function.} implies nearly optimal approximation for (multi-step) local search on $\mu$.

\begin{theorem}
\label{thm:main}
Consider a distribution $\mu: \binom{[n]}{k} \to \R_{\geq 0}.$ Suppose that for some $r = O(1)$, the $k \leftrightarrow (k-r)$ down-up random walk on $\lambda \ast \mu$ has spectral gap at least $k^{-O(1)}$ for all external fields $\lambda\in \R_{\geq 0}^n$. Then any approximate local maximum, that is a set $S\in \binom{[n]}{k}$ such that
\[ \mu(S)\geq \Omega(1)\cdot \max\set*{\mu(T)\given T\in \cN_r(S)} \]
is a $k^{O(k)}$-approximate global maximum, that is
\[ \mu(S)\geq k^{-O(k)}\cdot \max\set*{\mu(T)\given T\in \binom{[n]}{k}}.  \]
Moreover, such an approximate local maximum can be found efficiently given oracle access to $\mu$ and a starting point in the support of $\mu$.
\end{theorem}
In particular, combined with rapid mixing results of \cite{alimohammadi2021fractionally}, \cref{thm:main} implies that local search is an efficient $k^{O(k)}$-approximation algorithm for the optimization problem on nonsymmetric determinantal point processes (see \cref{subsec:nonsymDPP}), and on the intersection of a strongly Rayleigh distributions over $\binom{[n]}{k}$ and constantly many partition constraints (\cref{cor:strongly Rayleigh partition}). Our approximation algorithm for nonsymmetric determinantal point processes is the first unconditional multiplicative approximation algorithm for this problem.

\begin{remark}
    We remark that the approximation factor of $k^{O(k)}$ is nearly optimal amongst efficient algorithms. The special case of symmetric determinantal point processes was shown to be hard to approximate within a factor of $c^k$ for some constant $c>1$ \cite{CM10}. Further, the factor of $k^{O(k)}$ is tight for local search, even in the special case of symmetric determinantal point processes \cite[see, e.g.,][]{AV20}.
\end{remark}

\subsection{MAP Inference on Nonsymmetric DPPs} \label{subsec:nonsymDPP}
Determinantal point processes (DPPs) have found many applications in machine learning, such as data summarization \cite{gong2014largemargin,LB12}, recommender systems \cite{GartPK16,Wilhelm18}, neural network compression \cite{MS15}, kernel approximation \cite{LiJS16},
multi-modal output generation \cite{Elfeki19}, etc.

Formally, a DPP on a set of $[n]$ items is a probability distribution over subsets $Y\subseteq [n]$ parameterized by a matrix $L \in \R^{n\times n}$ where $Y$ is chosen with probability proportional to the determinant of the principal submatrix $L_Y$ whose columns are rows are indexed by $Y$: \[ \P{Y}\propto \det(L_Y).\]
A related and perhaps more widely used model, is a $k$-DPP, where the size of $Y$ is restricted to be exactly $k$. In applications, usually $k$ is set to be much smaller than $n$. We study this model in this paper. 

A fundamental optimization problem associated to probabilistic models, including DPPs, is to find the most likely, or the maximum a posteriori (MAP) configuration \cite{GKT12}:
\begin{equation} \label{eq:opt}
	\max\set*{\P{S}\given S\in \binom{[n]}{k}}.
\end{equation}

MAP inference is particularly useful when the end application requires outputting a single set; e.g., in recommender systems, the task is to produce a fixed-size subset of items to recommend to the user. 

Most prior work on DPPs requires the kernel matrix $L$ to be symmetric, but such symmetric kernels are known to be able to only encode repulsive (negatively correlated) interactions between items \cite{BBL09}. This severely limits their modeling power in practical settings. For example, a good recommender system for online shopping should model iPads and Apple Pencils as having positive interactions, since these are \emph{complementary} items and tend to be bought together.
To remedy this, recent work has considered the more
general class of nonsymmetric DPPs (NDPPs) and shown that these have additional useful modeling
power \cite{BWLGCG18,Gartrell2019LearningND}. \citet{Gartrell2019LearningND} consider NDPPs parameterized by nonsymmetric positive semi-definite (nPSD) kernel matrices $L$, i.e., those matrices where $L + L^\intercal \succcurlyeq 0$, and show efficient algorithms for learning such NDPPs.
\begin{definition}\label{def:nPSD}
A (not-necessarily-symmetric) matrix $L\in \R^{n\times n}$ is \textit{nonsymmetric positive semidefinite} (nPSD) if $L + L^\intercal \succeq 0$.  
\end{definition}
Throughout, we will consider only NDPPs with nPSD kernels (nPSD-NDPPs) \citep[see][for a survey on fixed-size DPPs and their applications]{KT12}. \citet{alimohammadi2021fractionally} showed how to efficiently sample from fixed-size nPSD-NDPPs using natural Markov chains. \citet{gartrell2020scalable} proposed a new learning algorithm, and showed how to efficiently implement and analyze the natural greedy MAP inference heuristic for symmetric DPPs on nPSD-NDPPs. This greedy heuristic (\cref{alg:common-greedy}) starts from an empty set and runs for $k$ iterations, in each iteration adding the item that most increases the DPP score.

Though this greedy algorithm is guaranteed to obtain a $k^{O(k)}$-approximation for symmetric DPPs \cite{CM10}, it could not achieve even a finite approximation factor for nPSD-NDPPs. For example, on a skew-symmetric matrix $X$, i.e., $X = -X^{\intercal}$, since all odd-sized principle minors of $X$ are zero, \cref{alg:common-greedy} would necessarily resort to picking an arbitrary/random item at every other iteration, which can result in an arbitrarily bad final answer. Consider a concrete example, which helps build intuition on why greedy fails to achieve a meaningful approximation factor. This example also shows that local search greedy \cite{KD16}\footnote{This algorithm starts with the output $S$ of \cref{alg:common-greedy}, then continuously swaps out an element in $S$ with one outside $S$ to increase the DPP score, until either a local maximum is reached or $k^2 \log k$ swaps have been performed.}, another candidate MAP inference algorithm with theoretical performance guarantees for symmetric DPPs, also used by \citet{gartrell2020scalable} as a baseline to compare their greedy method, also fails to achieve a meaningful approximation factor.  
\begin{example}\label{remark:greedy same as localSearch}
Consider $L$ composed of $2\times 2$ blocks $D_i = \begin{bmatrix} c_i& x_i\\ -x_i & c_i \end{bmatrix}$ where $c_i> 1$:
\[
    	L:=\begin{bmatrix}
    		D_1 & 0 & \dots & 0\\
    		0 & D_2 & \dots & 0\\
    		\vdots & \vdots & \ddots & \vdots \\
    		0 & 0 & \dots & D_n
    	\end{bmatrix},
    \]
We further assume that $c_1 > c_2 > \dots > c_n,$ $x_1 < x_2 < \dots < x_n$ and $x_i \gg c_j \; \forall i, j.$  It is easy to check that \cref{alg:common-greedy} (greedy) on input $k= 2t$ will select $ S = \set*{1,2 ,\dots, 2t-1, 2t}.  $ Indeed, \cref{alg:common-greedy} first picks item $1$ since $\det(L_{\set*{1}}) =c_1$ is maximum among all $L_{\set*{i}},$ then picks item $2$ since $ \det(L_{\set*{1,2}} ) = c_1^2 + x_1^2 >> \det(L_{\set*{1, i}}) = c_1 c_i \forall i \neq 1,$ and so on. On the other hand, the optimal subset is $\set*{n-2t+1, \dots, n}$ by our choice of $x$, and this could be arbitrarily better than \cref{alg:common-greedy}'s solution. We may think of items $2i- 1$ and $2i$ as complementary items, say, e.g., toothpaste and toothbrush proposed in a recommender system. The conditions on $c_i$'s and $x_i$'s mean that the degree of complementarity between these pairs increases with $i$. So $2n-1$ and $2n$ are the most likely pair to appear together, but each one of $2n-1$ and $2n$ is most unlikely to appear as a singleton, and the opposite holds for item $1$ and $2$; for example, think of $2n-1$ and $2n$ as a tea cup and tea cup lid, which are almost always bought together, but $1$ and $2$ as toothpaste and toothbrush, which are sometimes purchased separately.    

Furthermore, switching out any item $2 i-1 $ or $2i$ in $S$ for an item $2j-1$ or $2j$ outside of $S$ reduces the determinant by $ (c_i^2 + x_i^2) /c_i c_j > c_i >1,$ so $S $ is also maximum among its $1$-neighborhood. Thus local search greedy, or equivalently, local search initialized at $S$, will simply output $S$ itself. 
\end{example}
We remark that it is easy to construct an example where \cref{alg:common-greedy} produces a subset with zero determinant, whereas the optimal subset can have arbitrarily large determinant. E.g., in \cref{remark:greedy same as localSearch}, we can make all diagonal entries except for $L_{1,1}$ zero; then, \cref{alg:common-greedy} with even $k$ will necessarily produces a zero determinant.

As our main application, we show the first efficient algorithm for MAP inference on nPSD-NDPPs that gives a \emph{multiplicative} factor approximation for $\max\set{\det(L_{S})\given S\in \binom{[n]}{k}},$ without requiring any additional assumption on the kernel matrix $L$. Further, we obtain multiplicative approximation guarantees for $\det(L_{S})$, unlike prior related work \cite{gartrell2020scalable} which obtained multiplicative approximations for $\log \det(L_{S})$; this is often a stronger guarantee when $\OPT$ is sufficiently large -- roughly super-exponentially large in $k$. The assumptions behind prior work often implicitly imply that $\OPT$ is at least exponentially large in $k$, making our approximation guarantees attractive.
\begin{theorem}\label{thm:nonsym DPP main}
    	There is a polynomial time algorithm that on input $L\in \R^{n\times n}$ that is nPSD, outputs a set of indices $S\in \binom{[n]}{k}$ guaranteeing
    	\[ \det(L_{S})\geq k^{-O(k)}\cdot \max\set*{\det(L_{S})\given S\in \binom{[n]}{k}}. \]
    	Moreover, the algorithm runs in $O(n^4 k + n^2 k^5\log n) $ time given the entries of $L$, and $O(n^2 d^2 k + n^2 d^2 k^3 \log n)$ time given a rank-$d$ decomposition of $L$, i.e., $L = B C B^{\intercal}$ with $B\in \R^{n\times d}, C\in \R^{d\times d}.$
\end{theorem}
Our approximation factor matches that of the standard greedy heuristic on \textit{symmetric} DPPs, as well as the guarantee of other simple heuristics proposed for \textit{symmetric} DPPs \cite{CM10,KD16}. As mentioned earlier, \citet{CM10}'s greedy and \citet{KD16}'s local search algorithm do not achieve any finite approximation factor for nPSD-NDPPs. Our result is incomparable to \citet{gartrell2020scalable} as \begin{enumerate}[label=\roman*., series = tobecont, itemsep=0em, topsep=0em]
    \item  multiplicative
approximations for maximizing $\log \det(L_S)$ do not imply similar results for $\det(L_S)$,
    \item we place no additional assumption on $L$. As demonstrated earlier, our approximation guarantees hold for matrices $L$ where \cref{alg:common-greedy} fails to achieve even a finite approximation factor.
\end{enumerate}

Our local search algorithm for nPSD-NDPPs searches over $2$ neighborhoods, unlike most prior related works which typically use $1$ neighborhoods; using $2$ neighborhoods is necessary, and is compatible with intuition from prior work of \textcite{AV20} who first studied $2$ neighborhood local search for the related problem of finding the maximum $k\times k$ subdeterminant of a rectangular matrix. Unlike \cite{AV20}, our analysis of local search is not based on algebraic identities, which we believe do not have a counterpart in the world of nPSD-DPPs, but rather mixing properties of random walks.

\begin{Algorithm}
Initialize $S \gets \emptyset$.\;
\While{$\abs{S} < k$}{
    Pick $i\not\in S$ that maximizes $\det(L_{S\cup \set*{i}}),$ and update $S \gets S \cup \set*{i}$.\;
}
\caption{Standard GREEDY for DPPs} \label{alg:common-greedy}
\end{Algorithm}
\subsection{Composable Core-Sets for Strongly Rayleigh Distributions and Log-Concave Polynomials}
As further application of our methods, we extend prior work of \textcite{Mahabadi2019ComposableCF} on the construction of composable core-sets for maximizing symmetric DPPs to the more general class of distributions that satisfy the strongly Rayleigh property \cite{BBL09} or have a log-concave generating polynomial \cite{AOV18}.

Composable core-sets are a tool \cite{indyk2014composable} to handle computational problems involving large amounts of data. Roughly speaking, a core-set is a summary of a dataset that is enough to solve the computational problem at hand; a \emph{composable} core-set has the additional property that the union of summaries for multiple datasets is itself a good summary for union of all datasets. More precisely, in the context of the optimization problem on $\mu: \binom{[n]}{k}\to \R_{\geq 0}$, a function $c$ that maps any set $P\subseteq [n]$
to one of its subsets is called an $\alpha$-composable core-set (\cite{Mahabadi2019ComposableCF}) if it
satisfies the following condition: given any integer $m$ and
any collection of sets $P_1, \cdots , P_m \subseteq [n]$
\[\alpha \cdot \max\set*{\mu(S) \given S \subseteq \bigcup_{i=1}^m c(P_i)} \geq \max \set*{\mu(S) \given S \subseteq \bigcup_{i=1}^m P_i}. \]
We also say $c$ is a core-set of size $ t$ if $\card{c(P)} \leq t$ for all sets $P$. Composable core-sets are very verstaile; when a composable core-set is designed for a task, they automatically imply
efficient streaming and distributed algorithms for the same task.

One strategy for constructing composable core-sets is local search. \Textcite{Mahabadi2019ComposableCF} showed that for $k$-DPP parameterized by symmetric PSD matrix $L$, (1-step)-local seaarch (\cref{alg:localsearch} with $r=1$) gives a $k^{O(k)}$-composable core-sets of size $k.$ The approximation factor of $k^{O(k)}$ is nearly optimal. 

Recall that $k$-DPP parameterized by symmetric PSD matrix $L$ belongs to the family of homogeneous \emph{strongly Rayleigh} distributions, i.e., distributions $\mu$ whose generating polynomial $g_{\mu}$ is nonvanishing on the upper half plane \cite{BBL09}. An even more general family of distributions is the family of log-concave distributions \cite{AOV18}.
We extend \cite{Mahabadi2019ComposableCF}'s result to any distribution $\mu: \binom{[n]}{k} \to \R_{\geq 0}$ that is strongly Rayleigh or has a log-concave generating polynomial.

\begin{theorem} \label{thm:core-set}
Given a distribution $\mu: \binom{[n]}{k} \to \R_{\geq 0}$, let $c$ be a map that takes $P \subseteq [n]$ to some $c(P)\in \binom{P}{k}$ that is an $\zeta$-approximate local maximum in the $1$-neighborhood with respect to $\mu,$ for some fixed constant $\zeta \in (0,1)$:
\[ \mu(c(P))\geq \zeta \cdot \max\set*{\mu(S)\given S\in \cN_1(c(P))}. \]
Then $c$
is an $\alpha$-composable core-set 
 of size $k$ for the MAP-inference problem on $\mu$ with $\alpha = k^{O(k)}$ for strongly Rayleigh $\mu$, and $\alpha=2^{O(k^2)}$ when $\mu$ has a log-concave generating polynomial. 
\end{theorem}
\subsection{Techniques}
Our main tool for proving \cref{thm:main} is a form of (approximate) exchange inequality. Exchange inequalities have been traditionally been studied in discrete convex analysis \cite{Mur16}, but have recently been extended and used in sampling \cite{ALOVV21} and optimization \cite{AV20} problems beyond the reach of traditional discrete convex analysis. Unlike prior works, here we go in the opposite direction and show that efficient sampling implies a form of exchange inquality. To prove \cref{thm:main}, we set the external field $\lambda$ appropriately, and use the lower bound on the spectral gap of the down-up walk on $\lambda \ast \mu$ to derive our approximate exchange property \cref{lem:approx exchange FLC}. 

We then show that this approximate exchange property implies the desired approximation factor for local search (\cref{cor:local search guarantee FLC}).
Since nonsymmetric DPPs are $1/4$-fractionally log-concave \cite{alimohammadi2021fractionally},
\cref{thm:main} already implies an efficient algorithm (\cref{alg:localsearch} with $r=4$) to get $k^{O(k)}$-approximation factor for the MAP inference problem on nonsymmetric DPPs. We can further improve the the local search radius $r$ to $2$, and get a faster algorithm that matches the runtime stated in \cref{thm:nonsym DPP main} by showing a stronger approximate exchange property (\cref{thm:approximate_exchange}). 

To prove \cref{thm:core-set}, we use the approximate exchange property introduced by \cite{ALOVV21} that is satisfied by strongly Rayleigh and log-concave distributions. This exchange property is a quantitative version of the strong basis exchange axiom for matroids. We rename it the strong approximate basis exchange property (\cref{def:strong basis exchange}), to distinguish it from weaker exchange properties that we show in this paper. We show that the strong approximate basis exchange implies that approximate local maxima in the 1-neighborhood is a size-$k$ core-set with the desired approximation factor (\cref{lem:core-set}).

\section{Preliminaries}

We use $[n]$ to denote the set $\set{1,\dots,n}$ and $\binom{[n]}{k}$ to denote the family of size $k$ subsets of $[n]$. We use $\mathds{1}$ to denote the all $1$ vector. When $n$ is clear from context, we use $\mathds{1}_S\in \R^n$ to denote the indicator vector of the set $S\subseteq [n]$, having a coordinate of $0$ everywhere except for elements of $S$, where the coordinate is $1$. For sets $S, T$ of the same size we define their \emph{distance} to be
$ d(S, T): = \card{S\Delta T}/2 = \card{S \setminus T} = \card{T \setminus S}$.
With this notion of distance, we can define neighborhoods:
 \begin{definition}
    	For $r \geq 0$ let the $r$-neighborhood of $S\in \binom{[n]}{k}$ be 
    	\[ \cN_r(S) := \set*{T \in \binom{[n]}{k} \given d(S, T)\leq r }. \]
    \end{definition}

 For a density $\mu: 2^{[n]} \to \mathbb{R}_{\geq 0}$, the \emph{generating polynomial} of $\mu$ is defined as
	$$
	g_\mu(z_1, \ldots, z_n) = \sum_{S \in 2^{[n]}} \mu(S) \prod_{i \in S} z_i
	$$    
\subsection{Determinantal Point Processes (DPPs)}
A DPP on a set of $n$ items defines a probability distribution over subsets $Y \subseteq [n].$ It is parameterized by a matrix $L \in \R^{
n\times n}$: $\P_{L}{Y } \propto \det(L_Y ),$ where $L_Y$ denote the principle submatrix whose columns and rows are indexed by $Y.$ We call $L$ the kernel matrix. 

For $Y \subseteq [n]$, if we condition the distribution $\mathbb{P}_{L}$ on the event that items in $Y$ are included in the sample, we still get a DPP; the new kernel is given by the Schur complement $L^Y = L_{\tilde{Y}} - L_{\tilde{Y}, Y} L_{Y, Y}^{-1} L_{Y, \tilde{Y}} $ where $\tilde{Y} = [n] \setminus Y.$

Given a cardinality constraint $k$, the $k$-DPP paremeterized by $L$ is a distribution over subsets of size $k$ of $Y$ defined by $\mathbb{P}_{L}^k[Y] =\frac{ \det(L_Y) } {\sum_{\abs{Y'} =k } \det (L_{Y'})} . $

To ensure that $\mathbb{P}_L$ defines a probability distribution, all principal minors of $L$ must be
non-negative: $\det(L_S ) \geq 0.$ Matrices that satisfy this property are called $P_0$-matrices \citep[Definition 1]{FANG19891}. Any nonsymmetric (or symmetric) PSD matrix is automatically $P_0$-matrix \citep[Lemma 1]{Gartrell2019LearningND}.

We say a NDPP kernel $L\in \R^{n\times n}$ has a low-rank decomposition \cite{Gartrell2019LearningND,gartrell2020scalable} if $L$ can be written as $L = B C B^{\intercal}$ for some $d\leq n$, $B\in \R^{n\times d}, C\in \R^{d\times d}.$ Clearly, $\rank(L) = d$, and we say $L = B C B^{\intercal}$ is a rank-$d$ decomposition of $L.$ We will need the following identity, which is derived from Schur complements; it has previously appeared in \cite{gartrell2020scalable}. For $S \subseteq [n],$ let $B_S$ denote the sub-matrix of $B$ consisting of rows in $S$; then $L_S = B_S C B_S^{\intercal} $ and
\begin{equation} \label{eq:condition kernel}
    \begin{split}
        \det(L_{Y \cup D}) &= \det(L_{Y}) \det(  L_D - L_{D, Y} L_Y^{-1} L_{Y,D})\\
    	  &= \det(L_{Y})  \det(L_D - B_D C (B_Y^{\intercal} L_Y^{-1} B_Y) C B_D^T ).
    \end{split}
\end{equation}
Given $\det(L_Y)$ and $L_Y^{-1},$ we can compute $\det(L_{Y \cup D})$ in $O(\abs{D} d^2 + \abs{D}^2 d+ \abs{D}^3)$ time.

\subsection{MAP Inference}

Given a density $\mu:\binom{[n]}{k} \to \R_{\geq 0},$ the optimization with respect to $\mu$ or MAP inference on $\mu$ is to find
 \[S^*: = \arg\max_{S \in \binom{[n]}{k}} \mu(S).\]
 
Throughout the paper, we let $\OPT:= \max_{S\in \Omega} \mu(S).$ 

We say an algorithm gives a factor $c$-approximation for MAP inference on $\mu$ if its output $\hat{S} \in \binom{[n]}{k}$ such that $c \dot \mu(\hat{S}) \geq \OPT.$

When $\mu$ is defined by a DPP, i.e. $\mu(S) = \det(L_{S,S})$ for a $n \times n$ matrix $L,$ MAP inference on $\mu$ is also called the determinant maximization problem \cite[e.g., see][]{Mahabadi2019ComposableCF}.
\subsection{Markov Chains}
For two measures $\mu, \nu$ defined on the same state space $\Omega$, we define their total variation distance as
\[ \dtv(\mu, \nu)=\frac{1}{2}\sum_{\omega\in \Omega}\abs{\mu(\omega)-\nu(\omega)}=\max\set{\P_\mu{S}-\P_\nu{S}\given S\subseteq \Omega}. \]

A Markov chain on a state space $\Omega$ is defined by a row-stochastic matrix $P\in \R^{\Omega\times \Omega}$. We view distributions $\mu$ on $\Omega$ as row vectors, and as such $\mu P$ would be the distribution after one transition according to $P$, if we started from a sample of $\mu$. A stationary distribution $\mu$ for the Markov chain $P$ is one that satisfies $\mu P=\mu$. Under mild assumptions on $P$ (ergodicity), stationary distributions are unique and the distribution $\nu P^t$ converges to this stationary distribution as $t\to \infty$ \cite{LP17}. We refer the reader to \cite{LP17} for a detailed treatment of Markov chain analysis.


In this paper, we will only consider reversible Markov chain. We say a Markov chain with transition matrix $P$ is reversible if \[\mu(x) P(x,y) =\mu(y) P(y,x)\forall x, y\in \Omega.\]

The conductance\footnote{also known as bottleneck ratio in \cite{LP17}} of a subset $S$ of states in a Markov chain is
\[\Phi(S) = \frac{Q(S, \Omega \setminus S)}{\mu(S)}\]
where  $Q(S, \Omega\setminus S) = \sum_{x\in S, y \in \Omega\setminus S} \mu(x) P(x,y)$ is the ergodic flow between $S$ and $\Omega\setminus S,$ and $\mu(S) = \sum_{x\in S} \mu(x).$

The conductance of a Markov chain is defined as the minimum conductance
over all subsets $S$ with $\mu(S) \leq 1/2,$ i.e.
\[\Phi = \min_{S: \mu(S) \leq 1/2} \Phi(S)\]

\begin{theorem}[{\cite[see, e.g.,][Thm. 13.10]{LP17}}] \label{thm:cheeger}
 Let $\lambda_2$ be the second largest eigenvalue of the transition matrix $P$,  then
 \[ \frac{\Phi^2}{2}\leq 1-\lambda_2 \leq 2\Phi. \]
\end{theorem}
For a Markov chain $P$, we define the mixing time from a starting distribution $\nu$ as the first time $t$ such that $\nu P^t$ gets close to the stationary distribution $\mu$.
\[ \tmix(P, \nu, \epsilon)=\min\set{t\given \dtv(\nu P^t, \mu)\leq \epsilon}. \]
We drop $P$ and $\nu$ if they are clear from context. If $\nu$ is the Dirac measure on a single point $\omega$, we write $\tmix(P, \omega, \epsilon)$ for the mixing time. When mixing time is referenced without mentioning $\epsilon$, we imagine that $\epsilon$ is set to a reasonable small constant (such as $1/4$). This is justified by the fact that the growth of the mixing time in terms of $\epsilon$ can be at most logarithmic \cite{LP17}.

We can relate the mixing time and conductance as follow.
\begin{theorem}[{\cite[see, e.g.,][Thm. 7.4]{LP17}}]
For a reversible Markov chain $P$ with conductance $\Phi,$ we have 
\[ \tmix(P, 1/4) \leq 4 \Phi.\]
\end{theorem}
\subsection{The Down-Up Random Walk}
Consider a distribution $\mu: \binom{[n]}{k}\to \R_{\geq 0}.$ The down-up walk is given by the composition of two row-stochastic operators, known as the down and up operators.
\begin{definition}[Down Operator]
	For a ground set $\Omega$, and  $|\Omega| \geq k\geq \l$,  define the down operator $D_{k\to \l}\in \R^{\binom{\Omega}{k}\times \binom{\Omega}{\l}}$ as
	\[ 
		D_{k\to \l}(S, T)=\begin{cases}
			\frac{1}{\binom{k}{\l}}&\text{ if }T\subseteq S,\\
			0&\text{ otherwise}.\\
		\end{cases}
	\]
\end{definition}
Note that $D_{k\to \l}D_{\l\to m}=D_{k\to m}$. 

\begin{definition}[Up Operator]
	For a ground set $\Omega$, $|\Omega|\geq k\geq \l$, and density $\mu:\binom{\Omega}{k}\to \R_{\geq 0}$, define the up operator $U_{\l \to k}\in \R^{\binom{\Omega}{\l}\times \binom{\Omega}{k}}$ as
	\[ 
		U_{\l\to k}(T, S)=\begin{cases}
			\frac{\mu(S)}{\sum_{S'\supseteq T}\mu(S')}&\text{ if }T\subseteq S,\\
			0&\text{ otherwise}.\\
		\end{cases}
	\]
\end{definition}
If we define $\mu_k=\mu$ and more generally let $\mu_\l$ be $\mu_k D_{k\to \l}$, then the down and up operators satisfy
\[ \mu_k(S)D_{k\to \l}(S, T)=\mu_\l(T)U_{\l \to k}(T, S). \]
This property ensures that the composition of the down and up operators have the appropriate $\mu$ as a stationary distribution, are reversible, and have nonnegative real eigenvalues.
\begin{proposition}[{\cite[see, e.g.,][]{KO18,AL20,ALO20}}]
	The operators $D_{k\to \l}U_{\l\to k}$ and $U_{\l\to k}D_{k\to \l}$ both define Markov chains that are time-reversible and have nonnegative eigenvalues. Moreover $\mu_k$ and $\mu_\l$ are respectively their stationary distributions.
\end{proposition}

\begin{definition}[Down-Up Walk]
	For a ground set $\Omega$, $|\Omega|\geq k\geq \l$, and density $\mu:\binom{\Omega}{k}\to \R_{\geq 0}$, the $k\leftrightarrow \ell$ down-up walk is defined by the row-stochastic matrix $U_{\ell \to k}D_{k\to \ell}$.
\end{definition}

\subsection{Real-Stable and Sector-Stable Polynomials}
We use $\F[z_1,\dots,z_n]$ to denote $n$-variate polynomials with coefficients from $\F$, where we usually take $\F$ to be $\R$ or $\C$. We denote the degree of a polynomial $g$ by $\deg(g)$. We call a polynomial homogeneous of degree $k$ if all nonzero terms in it are of degree $k$.  
\begin{definition}[Stability]
	For an open subset $U\subseteq \C^n$, we call a polynomial $g\in \C[z_1,\dots,z_n]$ $U$-stable iff
	\[ (z_1,\dots,z_n)\in U\implies g(z_1,\dots,z_n)\neq 0. \]
	We also call the identically $0$ polynomial $U$-stable. This ensures that limits of $U$-stable polynomials are $U$-stable. For convenience, when $n$ is clear from context, we abbreviate stability w.r.t.\ regions of the form $U\times U\times \cdots \times U$ where $U\subseteq \C$ simply as $U$-stability.
\end{definition}

Our choice of the region $U$ in this work is the product of open sectors in the complex plane. 

\begin{definition}[Sectors]
	We name the open sector of aperture $\alpha\pi$ centered around the positive real axis $\Gamma_\alpha$:
	\[ \Gamma_\alpha:=\set{\exp(x+iy)\given x\in \R, y\in (-\alpha\pi/2,\alpha\pi/2)}. \]
\end{definition}


Note that $\Gamma_1$ is the right-half-plane, and $\Gamma_1$-stability is the same as the classically studied Hurwitz-stability \cite[see, e.g.,][]{Bra07}. Another closely related notion is that of real-stability where the region $U$ is the upper-half-plane $\set{z\given \Im(z)>0}$ \cite[see, e.g.,][]{BBL09}. Note that for \emph{homogeneous} polynomials, stability w.r.t.\ $U$ is the same as stability w.r.t.\ any rotation/scaling of $U$; so Hurwitz-stability and real-stability are the same for \emph{homogeneous} polynomials.

We use $\alpha$-sector-stable as a shorthand for $\Gamma_\alpha$-stable. Naturally, we call a distribution $\alpha$-sector-stable if its generating polynomial is $\alpha$-sector-stable. 

\begin{proposition}[{\citep{alimohammadi2021fractionally}}]\label{prop: ssProperties}
The following operations preserve $S_\alpha$-sector-stability on homogeneous multi-affine polynomials:
\begin{enumerate}[itemsep=0em, topsep=0em]
    \item\label{part: spec}Specialization: $g\mapsto g(a,z_2,\ldots,z_n)$, for $a\in \bar S_\alpha$.
    \item \label{part: derivative} Derivative: $g \mapsto \frac{\partial}{\partial z_1} g(z_1, \cdots, z_n) $.
    \item \label{part:scaling} Scaling: $g \mapsto g(\lambda_1 z_1, \ldots, \lambda_n z_n)$, for $\lambda \in \R_{\geq 0}^n.$
\end{enumerate}
\end{proposition}

We state some examples of sector stable distributions.

\begin{lemma}[{\citep{alimohammadi2021fractionally}}] \label{lem:p0Constrained}
Consider $L\in \R^{n \times n}$ that is nPSD, i.e., $L + L^T\succcurlyeq 0$, then $\mu: \binom{[n]}{k}\to \R_{\geq 0}$ defined by $\mu(S) = \det(L_{S,S})$ is $1/2$-sector-stable. 
\end{lemma}

\begin{lemma}[{\citep{alimohammadi2021fractionally}}] \label{lem:strong Rayleigh partition constraint}
Given a density $\mu:\binom{[n]}{k}\to \R_{\geq 0}$ and a partition $T_1\cup T_2\cup \cdots \cup T_s=[n]$, and numbers $c_1,\dots,c_s\in \Z_{\geq 0}$, let the partition constraint density $\mu_{T, c}$ be $\mu$ restricted to sets $S\in \binom{[n]}{k}$ where $\card{S \cap T_i} = c_i.$ When $\mu$ is strongly Rayleigh, $\mu_{T,c}$ is $1/2^c$-sector-stable. 
\end{lemma}
\subsection{Log-Concavity and Fractional Log-Concavity}

We now formally introduce \emph{log-concavity} for distributions over size-$k$ subsets of $n$ elements, and its direct generalization, \emph{$\alpha$-fractional-log-concavity}. 

\begin{definition}
	A function $f: \mathbb{R}_{\geq 0}^n \to \mathbb{R}_{\geq 0}$ is \emph{log-concave} if $\log f(z_1, \ldots, z_n)$ is concave over $\mathbb{R}_{\geq 0}^n$, i.e. for all $x, y \in \mathbb{R}_{\geq 0}^n$ and $\lambda \in [0, 1]$, we have:
	$$
	g(\lambda x - (1 - \lambda)y) \geq g(x)^\lambda \cdot g(y)^{1 - \lambda} \Longleftrightarrow \log g(\lambda x - (1 - \lambda)y) \geq \lambda \log g(x) + (1 - \lambda) \log g(y)
	$$
	We say a probability distribution $\mu: \binom{[n]}{k} \to \mathbb{R}_{\geq 0}$ is \emph{log-concave} if $\log g_\mu(z_1, \ldots, z_n)$ is concave over $\mathbb{R}_{\geq 0}$, or in other words, that its generating polynomial is a log-concave function over $\mathbb{R}_{\geq 0}$. 
\end{definition}

\cite{ALOV19,BH19} shows that for homogeneous multiaffine polynomials, real-stability implies log-concavity. A similar relationship holds for sector stability and fractional log-concavity. 
\begin{lemma} \label{lem:sector-stability-to-log-concavity}
\emph{(Lemma 67 from \cite{alimohammadi2021fractionally})} If a polynomial $g$ is $\alpha$-sector-stable, then it is $\frac{\alpha}{2}$-fractionally-log-concave.
\end{lemma}
We note that scaling preserves $\alpha$-log-concavity 
of homogeneous distributions \cite{alimohammadi2021fractionally} i.e. if $\mu$ is $\alpha$-log-concave, then so is $\lambda \ast \mu$ for all $\lambda \in \R_{\geq 0}^n.$ 
\begin{theorem}[{\cite{alimohammadi2021fractionally,AJKPV21}}] \label{thm:spectral gap FLC}
    Suppose $\mu: \binom{[n]}{k}$ is $\alpha$-fractional-log-concave. The $k \leftrightarrow (k-\lceil 1/\alpha \rceil)$-down-up-walk on $\mu$ has spectral gap at least $\Omega(k^{-1/\alpha}).$
\end{theorem}
\subsection{Composable Core-Set}
\begin{definition}[{\cite[Definition 2.2]{Mahabadi2019ComposableCF}}] \label{def:composable core-set}
A function $c(P)$ that maps the input set $P \subseteq \R^d$ to one of its subsets is called an $\alpha$-composable core-set for a function $f:2 ^{\R^d} \to \R$ if, for any collection of sets $P_1, \ldots, P_n\subseteq \R^d$ we have $f(C) \geq f(P)/\alpha$ where $P = \bigcup_{i\leq n} P_i$ and $C= \bigcup_{i\leq n} c(P_i)$   
\end{definition}

\section{MAP Inference via Local Search} \label{sec:implement}
In this section, we show how to efficiently find a local optima\footnote{More precisely, we show how to find an approximate local optima, which is sufficient for our purpose.} of a given distribution $\mu.$ We run a two stage algorithm: 
\begin{enumerate}[label = (\roman*)]
    \item first, we find some ``good'' initial subset $S_0 \in \binom{[n]}{k}$, i.e., one such that the ratio $\OPT/\mu(S_0)$ is bounded by $2^{\poly(n,k)}$ (see \cref{lem:crude}),
    \item \label{step: local search} then, for a suitably chosen radius $r \in \N_{\geq 1}$, we run a simple local search (\cref{alg:localsearch}) that starts with $S\leftarrow S_0$, and find better and better solutions by swapping \emph{at most $r$} elements in $S$ for elements outside of $S$ 
    until no more improvement in term of $\mu(S)$ can be found.
\end{enumerate}
To ensure that our algorithm terminates within polynomial time, we will only take improvements that increase the determinant by at least a lower multiplicative threshold, say, by a factor of $2$.
\begin{Algorithm} 
\textbf{Input:} $ \alpha \leq 1, S_0 \in \binom{[n]}{k}$ with $\mu(S_0) > 0.$\;
Initialize $S \leftarrow S_0$.\;
\While{$\mu(S) < \zeta\cdot  \mu(T)$ for some $T \in \cN_r(S)$}{
 Update $S \leftarrow \arg\max_{T\in  \cN_r(S)} \mu(T)$.\; 
}
\caption{LOCAL-SEARCH-$r$ ($\LS_r$)} \label{alg:localsearch}
\end{Algorithm}

We prove the algorithmic part of \cref{thm:main}, 
that with a suitable choice for $S_0$, \cref{alg:localsearch} runs in polynomial time.
    \begin{proposition}\label{prop:steps-bound}
    	The number of steps taken by \cref{alg:localsearch} with $r = O(1)$ starting from $S_0$ is at most
    	\[ \log_{1/\alpha}\parens*{\OPT/\mu(S_0)}. \]
    	Each step can be implemented using $O((nk)^r)$ oracle access to $\mu.$
    	
    \end{proposition}
    \begin{proof}
    	Each iteration improves $\mu(S)$ by a factor of at least $1/\alpha$. On the other hand, this value can never exceed $\OPT$, and it starts as $\mu(S_0) > 0$.
    	
    	Clearly, to perform local search in the $r$-neighborhood of a set $S,$ we only need to query $\mu((S \setminus U_1) \cup U_2)$ for $U_1 \in \binom{S}{\leq r}$ and $U_2 \in \binom{[n]}{\leq r}.$  The total number of such queries is $ O((nk)^r).$
    
    \end{proof}

    
\begin{definition}
    	For $\mu: \binom{[n]}{k}\to \R_{\geq 0}$ and $\alpha > 0,$ we say $S\in \binom{[n]}{k}$ is a $(r,\zeta)$-local maximum w.r.t. $\mu$ if 
    	$\mu(S) \geq \zeta \mu(T)$  for all $T\in \cN_r(S)$. 
    \end{definition}
Clearly, when \cref{alg:localsearch} terminates, the output is a $(r,\zeta)$-local maximum.

    
    Next, we show how to obtain a ``good'' initialization $S_0$ by a simple greedy algorithm, which we call INDUCED-GREEDY, that is based on maximizing the \textit{marginal gain} defined by the distribution on size $\leq k$ subsets. This gain is \emph{induced}  by the distribution $\mu,$ as defined below.  
    
    For subset $T$ of $[n]$ of size $\leq k$, let $\mu(T) = \sum_{S \in  \binom{[n]}{k}: S \supseteq T} \mu(S).$ 
    \begin{Algorithm}
        Initialize $S \leftarrow \emptyset$.\;

    \While{$\abs{S} < k$}{
        Pick $i\not\in S$ that maximizes $\mu(S \cup \set*{i})$ and update $S \leftarrow S \cup \set*{i}$.\;
    }
    \caption{INDUCED-GREEDY} \label{alg:greedy}
    \end{Algorithm}
    \begin{lemma}\label{lem:crude}
    	\cref{alg:greedy} 
    	returns $S_0$ with
    	\[ O(n^k) \cdot \det(L_{S_0})\geq \OPT. \]

    \end{lemma} 
    \begin{proof}
For $j\in [k],$ let $i_j$ be the element added to $S$ at the $j$-th iteration of the while loop. Let $S_0 = \emptyset$, $S_j = S_{j-1} \cup \set*{i_j}.$ Observe that $\abs{S_j} =j$ and for each $j\geq 0$
\begin{align*}
  \mu(S_j) &=  \frac{1}{k - \abs{S_j} } \sum_{i\not\in S_j} \mu(S_j\cup \set*{i}) \leq  \frac{n - j }{k - j } \mu(S_{j+1}) 
\end{align*}
thus $\binom{n}{k}\mu(S_k) \geq \mu(S_0) = \mu(\emptyset) = \sum_{S'\in \binom{[n]}{k}} \mu(S') \geq \OPT$.

\end{proof}
\begin{remark} \label{remark:greedy runtime}
In \cref{alg:greedy}, it is enough to find $i$ that approximately maximizes $\mu(S\cup {i})$ i.e. for some constant $\zeta \in (0,1)$, $\mu(S\cup {i}) \geq \zeta \mu(S \cup {j}) $ for all $j\not \in S.$ In that case, \cref{lem:crude} still holds, and \cref{alg:greedy} can be efficiently implemented given access to efficient algorithms that approximately sample from $\lambda \ast \mu$ for $\lambda \in \R_{\geq 0}^n.$ Indeed, note that $\mu(S\cup {i})/\mu(S)$ is the marginal of $\lambda \ast \mu$ where $\lambda_i=\begin{cases} \infty &\text{ for } i \in S \\1 &\text{ else} \end{cases}.$ Thus, $\mu(S\cup {i})$ can be approximate within some small constant factor. 
\end{remark}

\section{From Sampling to Optimization via Local Search}
In this section, we prove \cref{thm:main}.
\begin{definition}[$r$-exchange]
For $\mu: \binom{[n]}{k} \to \R_{\geq 0}$, $r \in \N$ and $S, T \in  \binom{[n]}{k}$, we let \[\cE^r(S,T) : = \set*{U \subseteq S\Delta T \given \abs{U \cap S } =  \abs{U \cap T }    =r}  \] be the set of all $r$-exchanges between $S$ and $T.$ 
\end{definition}


\begin{definition}[Weak $(r,\beta)$-approximate exchange]
We say a distribution $\mu: \binom{[n]}{k}$ satisfies \emph{weak $(r,\beta)$-approximate exchange} if for any $S, T \in \binom{[n]}{k},$ there exists $s \in \set*{1,\cdots,r}$ and $U \in \cE^{s}(S,T) $  such that
\[\mu(S) \leq \beta\cdot \mu(S \Delta U) \parens*{\frac{\mu(S)}{\mu(T)}}^{s/d(S,T)}\]
\end{definition}

\begin{lemma}\label{lem:approx exchange FLC}
Consider $\mu:\binom{[n]}{k} \to \R_{\geq 0}$ such that for all external field $\lambda\in \R_{\geq 0}^n,$ the conductance of the $k\leftrightarrow (k-r)$-down-up walk on $\lambda \ast \mu$ is at least $\Omega(k^{-c}).$ Then $\mu$ satisfies weak $(r, O(k^{r+c}))$-approximate exchange.

\end{lemma}
\begin{proposition} \label{cor:local search guarantee FLC}
If  $\mu:\binom{[n]}{k} \to \R_{\geq 0}$ satisfies weak $(r,\beta)$-approximate exchange then any $(r,\zeta)$-local max with respect to $\mu$ is also an $O((\beta/\zeta)^k)$-approximate global max. 
\end{proposition}
In particular, when $\mu$ is $\alpha$-fractionally log-concave, \cref{lem:approx exchange FLC,cor:local search guarantee FLC} hold with $r = \lceil1/\alpha \rceil$ and $ c= 1/\alpha$ and $\beta = O(k^{r+c}).$

\begin{corollary}\label{cor:strongly Rayleigh partition}
Let $\mu:\binom{[n]}{k} \to \R_{\geq 0}$ be strongly Rayleigh. Given a density $\mu:\binom{[n]}{k}\to \R_{\geq 0}$ and a partition $T_1\cup T_2\cup \cdots \cup T_s=[n]$, and numbers $c_1,\dots,c_s\in \Z_{\geq 0}$, let the partition constraint density $\mu_{T, c}$ be $\mu$ restricted to sets $S\in \binom{[n]}{k}$ where $\card{S \cap T_i} = c_i.$ When $\mu$ is strongly Rayleigh and $c=O(1)$, one can efficiently finds a $k^{O(k)}$-approximation for $\max \mu_{T,c} (\cdot)$ using \cref{alg:localsearch} with $r = 2^c.$
\end{corollary}
The local search guarantee in \cref{thm:main} follows from \cref{thm:spectral gap FLC,thm:cheeger,lem:approx exchange FLC,cor:local search guarantee FLC}, and the runtime bound follows from \cref{remark:greedy runtime,prop:steps-bound}.

    

\begin{proof}[Proof of \cref{lem:approx exchange FLC}]
If $d(S,T) \leq r$ then the lemma holds trivially by setting $U = S \Delta T.$ In what follows, we assume $d(S,T) \geq r.$
Wlog we can assume that $S = \set*{1, \dots, t} \cup C$ and $T = \set*{t+1, \dots, 2t} \cup C$ with $C = \set*{2t+1, \dots, t+k}$ and $t = d(S,T).$ 

Consider distribution $\mu' = \lambda \ast \mu$ with $\lambda_i = \begin{cases} 1 &\text{ if } 1\leq i \leq t\\
(\mu(S)/\mu(T))^{t}  &\text{ if } t+1\leq i \leq 2t \\
 \infty &\text{ if } 2t+ 1 \leq i \leq t+k\\
0  &\text{ else }
\end{cases}.$

Note that $\mu'$ is supported on $W\in \binom{[n]}{k}$ where $(S \cap T) = C \subseteq W \subseteq (S \cup T).$
Let $\Phi$ be the conductance of the $k\leftrightarrow (k-r)$-down-up  walk on $\mu'$, then $\Phi \geq \Omega(k^{-c}).$ On the other hand, since $\mu'(S) = \mu'(T) = \mu(S)\leq \frac{\sum_{W}\mu'(W)} {2},$ we have that by definition of $\Phi$
\[ \Phi=\min_{\mu'(\mathcal{S})\leq \mu'(\Omega)/2 } \frac{Q (\mathcal{S}, \Omega \setminus \mathcal{S} ) }{\mu'(\mathcal{S}) } \leq \frac{ Q(\set*{S}, \Omega \setminus \set*{S}) }{\mu'(S) }\]
where we can rewrite $Q(\set*{S}, \Omega \setminus \set*{S})$ as 
\[Q(\set*{S}, \Omega \setminus \set*{S}) = \mu'(S) \frac{1}{\binom{k}{r}} \sum_{U_1 \in\binom{S}{r}} \sum_{\substack{W \supseteq S \setminus U_1\\W \in \supp(\mu')\setminus \set*{S}}} \frac{\mu'(W)}{\mu'(S \setminus U_1)}\]
where $\mu'(S \setminus U_1) = \sum_{W \in \binom{[n]}{k}, W \supseteq S \setminus U_1} \mu'(W).$

Note that \[ \set*{W  \in \supp(\mu') \setminus \set*{S} \given W \supseteq S \setminus U_1  } \subseteq \set*{(S \setminus U_1 ) \cup U_2 \given  U_2 \in \binom{T \cup U_1}{r} \setminus \set*{U_1}}\]
thus \[\bigcup_{U_1 \in \binom{S}{r}} \set*{W  \in \supp(\mu') \setminus \set*{S} \given W \supseteq S \setminus U_1  } \subseteq \set*{S \Delta U \given U \in \bigcup_{s=1}^r \cE^s(S, T)} .\] Moreover, $\abs{\binom{S}{r}} = \binom{k}{r}$ and for each $U_1  \in\binom{S}{r},$ the cardinality of $  \set*{W  \in \supp(\mu') \setminus \set*{S} \given W \supseteq S \setminus U_1  } $ is at most $ \leq \binom{k+r}{r}-1\leq k^r.$ 
Hence, there must exist $r \in [s]$ and $U \in \cE^s(S, T)$ such that
\[\frac{\mu'(S\Delta U)}{\mu'(S\setminus U)} \geq \frac{1}{k^r} \frac{ Q(\set*{S}, \Omega \setminus \set*{S}) }{\mu'(S) } \geq \Omega(k^{-(r+c)}) .   \]
Thus
\[\mu(S) = \mu'(S) \leq \mu'(S \setminus U) \leq O(k^{r + c}) \mu'(S \Delta U) = O(k^{r + c}) \mu(S \Delta U)(\frac{\mu(S)}{\mu(T)})^{s/t}.  \]
\end{proof}
\begin{proof}[Proof of \cref{cor:local search guarantee FLC}]
Apply \cref{lem:approx exchange FLC} for $S$ being a $(r,
\zeta)$-local max and $T: = \arg\max \mu(W).$ Let $t = d(S,T).$ For some $s \in [r]$ and $U \in \cE^s(S,T)$ 
\[\mu(S) \leq O(k^{r+c}) \mu(S \Delta U) (\frac{\mu(S)}{\mu(T)})^{s/t} \leq O(k^{r+c}/\zeta) \mu(S) (\frac{\mu(S)}{\mu(T)})^{s/t}  \]
where the inequality follows from definition of $(r,\zeta)$-local max. Divide both sides by $\mu(S) > 0,$ we get
\[\mu(T) \leq O(k^{r+c}/\zeta)^{t/s} \mu(S) \leq O(k^{r+c}/\zeta)^k \mu(S). \]
where we use the fact that $t/s \leq k.$
\end{proof}


\section{Improved Local Search for Sector-Stable Distributions}
By \cref{lem:sector-stability-to-log-concavity,lem:approx exchange FLC}, for any $\alpha$-sector-stable distribution $\mu:\binom{[n]}{k},$ \cref{alg:localsearch} with $r = \lceil \frac{2}{\alpha} \rceil$ finds a $k^{O(k/\alpha)}$-approximation of $\OPT.$ In this section, we show how to improve the local search radius $r$ to $\lceil \frac{1}{\alpha} \rceil$ for $\alpha \in [1/2,1].$ As an application, we prove \cref{thm:nonsym DPP main}. 

When $\alpha =1,$ the distribution $\mu$ is real stable, thus log-concave, and \cref{cor:local search guarantee FLC} already shows $\LS_1$ gives a $k^{O(k)}$-approximation for MAP inference. Clearly, for $\alpha \in [1/2,1),$ any $\alpha$-sector-stable $\mu$ is also $1/2$-sector-stable, and $\lceil 1/\alpha \rceil = 2,$ so we only need to consider the case $\alpha =1/2.$

\begin{definition}[$(r, \beta)$-approximate exchange]\label{def:approx_exchange}
For $r\in \N_{\geq 1}$ and $\beta > 0$, we say $\mu:  \binom{[n]}{k} \to \R_{\geq 0} $ satisfies $(r, \beta)$-approximate exchange if for any $S, T \in \binom{[n]}{k}$
\[\mu(S) \mu(T) \leq \max_{i=1}^r \set*{\beta^i M^i(S \to T) M^i(T \to S) }\]
where $M^i (S \to T) : = \max_{U \in \cE^i (S, T)}\mu(S \Delta U)$.
\end{definition}
\begin{theorem}\label{thm:approximate_exchange}
Suppose $\mu: \binom{[n]}{k} \to \R_{\geq 0}$ is $1/2$-sector stable. For any $S, T \in  \binom{[n]}{k}$
\begin{align}\label{ineq:exchange}
    \mu(S) \mu(T) &\leq \max_{i=1}^2 \set*{\parens*{\sum_{U \in \cE^i(S,T)} \mu(S \Delta U)}\parens*{\sum_{U \in \cE^i(S,T)} \mu(T \Delta U)}}\nonumber \\
    &\leq \max_{i=1}^2 \set*{k^{4i} M^i(S \to T) M^i(T \to S) }.
\end{align}
Consequently, $\mu$ satisfies $(2, k^4)$-approximate exchange.
\end{theorem}
We prove the approximate exchange property by relying on the following theorem.
\begin{theorem}[\cite{asner70}] \label{thm:hurwitz minor}
	Consider a univariate $1$-sector-stable (Hurwitz-stable) polynomial $f(z) = \sum_{i=0}^n a_i z^i$ with $a_i \geq 0 \forall i.$ Its Hurwitz matrix $H = (h_{ij}) \in \R^{n\times n}$ is defined by $h_{ij} = a_{2j-i}$ when $0 \leq 2j-i \leq n$, otherwise $h_{ij}=0$. $H$ is totally nonnegative, in the sense that all its minors are nonnegative.
	\end{theorem}
	As an immediate consequence, we obtain the following lemma about coefficients of  univariate Hurwitz stable polynomial.
\begin{lemma} \label{lemma:hurwitzCoeff}
If $f(z) = a_n z^n + \cdots + a_1 z + a_0$ with $a_i\geq 0\forall i$ is $1$-sector stable,then $a_n a_0 \leq \max \set*{a_1 a_{n-1},  a_2 a_{n-2}}$ 
\end{lemma}
	\begin{proof}
	If $n\leq 2$ then the claim is trivially true. Below, we assume $n\geq 3.$
	We consider two cases, when $n$ is odd and when $n$ is even.
	Suppose $n  = 2t -1$ for $t\in \N.$ 
    	    By \cref{thm:hurwitz minor}, all minors of $H$ are non-negative, hence 
    	    \begin{align*}
    	   \det\begin{bmatrix}
    	    h_{1,1} & h_{1, t}\\
    	    h_{2,1} & h_{2, t}
    	    \end{bmatrix} &= \det \begin{bmatrix}
    	    a_1 & a_{2t-1}\\
    	    a_0 & a_{2t-2}
    	    \end{bmatrix} = a_1 a_{2t-2} - a_0 a_{2t-1} \\
    	    &= a_1 a_{n-1} - a_0 a_n  \geq 0.
\end{align*}
Suppose $n = 2t$ for $t\in \N.$ Again, \cref{thm:hurwitz minor} implies  \begin{align*}
    	   \det\begin{bmatrix}
    	    h_{2,2} & h_{2, t+1}\\
    	    h_{4,2} & h_{4, t+1}
    	    \end{bmatrix} &= \det \begin{bmatrix}
    	    a_2 & a_{2t}\\
    	    a_0 & a_{2t-2}
    	    \end{bmatrix} = a_2 a_{2t-2} - a_0 a_{2t} \\
    	    &= a_2 a_{n-2} -a_0 a_n \geq 0.\qedhere
\end{align*}
\end{proof}
\cref{lemma:hurwitzCoeff}, in turn implies the following fact about coefficients of $1/2$-sector-stable univariate polynomials that only have even-degree terms.
\begin{corollary} \label{cor:S[1/2]Coeff}
If $f(z) = \sum_{i=0}^t a_{2i}z^{2i}$ with $a_{2i}\geq 0 \forall i$ is $1/2$-sector stable, then $a_0 a_{2t} \leq \max \set*{a_2 a_{2t-2}, a_4 a_{2t-4}} $
\end{corollary}
\begin{proof}
Let $g(z) = f(z^{1/2}) = \sum_{i=0}^t a_{2i} z^i$ then $g(z)$ is $1$-sector stable, and the claim follows from \cref{lemma:hurwitzCoeff}.
\end{proof}
\begin{proof}[Proof of \cref{thm:approximate_exchange}]
    Let $f (z_1, \cdots, z_n) = \sum_{W \in \binom{[n]}{k}}\mu(W) z_W$ be the generating polynomial of $\mu.$ We deal with the case where $S\cap T = \emptyset$ and $[n] = S \cup T.$ Other cases can be reduced to this scenario by setting $z_i$ to $0$ for $i\not \in S \cup T$, and taking derivative(s) of $f$ with respect to $i\in S\cap T.$ Recall that setting variables to $0$ and taking derivative(s) preserve $1/2$-sector-stability and homogeneity of polynomials (see \cref{prop: ssProperties}).  
W.l.o.g., assume $S = [t]$ and $T = \set*{t+1, \cdots, 2t}.$ We can rewrite $f$ as
$f(z_1, \cdots, z_{2t}) = \sum_{W \in \binom{[2t]}{t}} \mu(W) z_W$. 

In $f$, set $z_i = z$ if $i\in S$ and $z_i = z^{-1}$ if $i\in T.$ We obtain a single variate $1/2$-stable polynomial \[\tilde{f}(z) = z^t f (z, \cdots, z, z^{-1}, \cdots, z^{-1}) = \sum_{i=0}^{2t} b_i z^i  = \sum_{i=0}^t b_{2i} z^{2i} \]
 Note that a term $\mu(W) z^W$ contribute to $b_i$ if and only if $i = \abs{W \cap S} + (t - \abs{W\cap T}) = 2 \abs{W\cap S}  .$ In particular, $b_{2i+1} = 0$ for all $i\in  \N$ and
 \begin{align*}
 b_{2i} &= \sum_{W: \abs{W \cap S} = i}\mu(W) = \sum_{U \in \cE^i(S,T)} \mu(T \Delta U) 
 \\
 &= \sum_{U \in \cE^{t-i}(S,T)} \mu(S \Delta U).    
 \end{align*}
 In particular, $b_{2t} = \mu(S)$ and $b_0 = \mu(T).$
 The first line of \cref{ineq:exchange} follows by applying \cref{cor:S[1/2]Coeff} to $\tilde{f}$, and the second line follows by observing that $ \sum_{U \in \cE^i(S,T)} \mu(T \Delta U) \leq \binom{\abs{S \Delta T} /2}{i}^2 \max_{U \in \cE^i(S,T)} \mu(T \Delta U) $.
\vskip -0.25in
\end{proof}
\begin{lemma}
\label{lem:local-to-global_gen}
Suppose $\mu: \binom{[n]}{k} \to \R_{\geq 0}$ satisfies $(r, \beta)$-exchange and $S$ is a $(r,\zeta)$ local maximum with $\alpha \leq 1$ and $\mu(S) > 0$. Then $S$ is a $(\beta/\zeta)^{ k}$-approximate global optimum:
\[  (\beta/\zeta)^{k}\mu(S)\geq  \max_{T\in \binom{[n]}{k} } \mu(T).\]
\end{lemma}
\begin{proof}
	Let $S\in \binom{[n]}{k}$ be a $(r,\alpha)$-local maximum with $\mu(S) > 0$ and let $S^* : =\arg\max_{T \in \binom{[n]}{k}} \mu(T)$. We first prove the following claim.
	\begin{claim} \label{clm:reduceDistance}
	For any $T \in \binom{[n]}{k}$ where $T \neq S$, 
	there exists $i\in [r]$ and $W \in \binom{[n]}{k}$ such that $d(S, W) = d(S, T) -i$ and
	$
		\mu(T) \leq \frac{\beta^i}{\zeta} \cdot \mu(W).
	$
	\end{claim}
	\begin{proof}[Proof of \cref{clm:reduceDistance}]
	By \cref{def:approx_exchange}, for some $i\in [r]$, there exists $U_1, U_2 \in \cE^i(S, T)$ such that
	\begin{align*}
	    \mu(S) \mu(T) \leq \beta^{i} \mu(S \Delta U_1) \mu(T\Delta U_2) \leq \beta^i \frac{\mu(S)}{\zeta} \cdot \mu(T \Delta U_2) 
	\end{align*}
	   where the last inequality follows from the definition of $(r,\zeta)$-local maximum.
	   
	   Note that $d(S, T\Delta U_2) = d(S, T) - i.$ Setting $W =  T\Delta U_2 $ and dividing both sides by $\mu(S)> 0$ gives the desired inequality. 
	\end{proof}
	Note that initially $d(S, S^*) \leq k$. We can iteratively apply \cref{clm:reduceDistance} for up to $k$ times to obtain the desired inequality. Indeed, let $T_0 = S^*$, and for $j\geq 1$ let $i_j\in [r]$ and $T_j\in \binom{[n]}{k}$ be such that $\mu(T_{j-1}) \leq\frac{\beta^{i_j}}{\zeta}\cdot T_j $ and $d(S, T_j) = d(S, T_{j-1}) - i_j.$ \cref{clm:reduceDistance} guarantees the existence of such $i_j$ and $T_j,$ as long as $T_{j-1} \neq S.$ Let $s$ be the minimum index such that $d(S,T_s) = 0.$ Note that $s \leq k$ and $T_s = S.$ We have
	\begin{align*}
	  \mu(S^*) = \mu(T_0) \leq \frac{\beta^{i_1}}{\zeta}\cdot T_1  &\leq \frac{\beta^{i_1}}{\zeta}\cdot  \frac{\beta^{i_2}}{\zeta}\cdot T_2 \leq \cdots\\ \leq
	  &\prod_{j=1}^s \frac{\beta^{i_j}} {\zeta} \cdot \mu(T_s) \leq \frac{\beta^k}{\zeta^k} \mu(S)    
	\end{align*}
	
	where the last inequality follows from the facts that $\sum_{j=1}^s i_j = d(S,T_0) - d(S,T_s) \leq k$ and $(\frac{1}{\zeta})^s \leq (\frac{1}{\zeta})^k.$ 
\end{proof}
Now, we are ready to prove \cref{thm:nonsym DPP main}.
\begin{proof}[Proof of \cref{thm:nonsym DPP main}]
We let $\mu(S) = \det(L_S)$ and run the two stage algorithm in \cref{sec:implement} with $r=2.$ The approximation guarantee is a direct consequence of \cref{lem:p0Constrained,thm:approximate_exchange,lem:local-to-global_gen}.

Suppose we are given access to the entries of $L.$ Each iteration of \cref{alg:localsearch} clearly runs in $O(n^2 k^5)$ time, since $\cN_2(S)$ has at most $O(k^2 n^2)$ elements and  computing the determinant of $k\times k$ matrices costs $O(k^3)$ time. The cost of $\LS_2$ can be reduced to $O(n^2 k^4)$ time using Schur complements to compute all $\det(L_{Y \cup D})$ for each fixed $Y$ and all $D$ of size $\leq r $ in $O(k^3 + n^2 k^2)$ time.
    	\citep[see \cref{eq:condition kernel} or ][for example]{gartrell2020scalable}.
    	If we are only given $B, C$, then each of these submatrices and their determinant can be computed in $O(d^2)$ time, so that each iteration takes $O(n^2 d^2 k^2)$ time.
    	
    	Now, we bound the runtime of \cref{alg:greedy}.
    	To implement each iteration of \cref{alg:greedy}, we need to compute $\mu(Y) = \sum_{S \in\binom{[n]}{k}: S\supseteq Y} \det(L_Y) $ , which is the coefficient of $\lambda^{n-k}$ in $g(\lambda) = \det(L + \lambda\cdot \diag{\mathds{1}_{\tilde{Y}} })$ where $\tilde{Y} = [n] \setminus Y.$ 

There are several ways to compute $\mu(Y).$
To compute the coefficients of polynomial $g(\lambda)$ of degree $\leq n$, we can evaluate $g$ at $n+1$ distinct points $\lambda$ and use polynomial interpolation, i.e., solve a linear system of equations involving the the Vandermonde matrix. 
A more efficient way, which costs $O(n^3)$ per computation of $\mu(T)$, for a total runtime of $O(n^4 k)$, is as follow:
\begin{enumerate}[label = (\roman*)]
    \item Let $D=\diag{\mathds{1}_{\tilde{Y}} }. $  We use the QZ decomposition algorithm \citep[Section 7.7, p. 313]{QZDecomp} to compute unitary matrices $Q$, $Z$ such that
    \[ L = Q \tilde{A} Z^*, D =  Q \tilde{D} Z^* \]
     where $\tilde{A}, \tilde{D}$ are both upper triangular. 
    Note that $\deg(g) \leq n - \abs{T}.$
    
    Compute the roots of $g(\lambda) =\det(L + \lambda D),$ which are exactly the generalized eigenvalues $ \lambda_1, \dots,\lambda_{\deg(g)} $ defined by $\lambda_i = \frac{\tilde{A}_{i,i} }{ \tilde{B_{i,i}}}$ where we may assume w.l.o.g. that $\tilde{D}_{i,i} \neq 0$ for $i = 1, \dots, \deg(g)$, and is zero otherwise. Let $c : = \prod_{i \in [n]: B_{i,i} \neq 0 } \tilde{D}_{i,i} \prod_{i \in [n]: D_{i,i} = 0 } \tilde{A}_{i,i} $. Then
    \[g (\lambda) = c \prod_{i \in [\deg(g)]} (\lambda - \lambda_i) \]

    \item We then compute the $(k-n+n')^{th}$-symmetric polynomial of $ \lambda_1, \dots,\lambda_{n - \abs{T}} $ where
    $e_t = \sum_{W \in\binom{[k - \abs{T}] }{ t} } \prod_{j\in W} \lambda_j$ using the recursion \cite{KT12}
    \[ te_t = e_{t-1}p_1 - e_{t-2}p_2 + e_{t-3}p_3 -\dots \pm p_k  \]
    with $p_t = \sum \lambda_j^t,$ and output $ \mu(Y) =  \abs{c  e_{k-(n - \deg(g) )}} $.
\end{enumerate}
Given the low-rank decomposition $L = B C B^{\intercal},$ we can further optimize by reducing the cost of step (i) to $O(nd^2)$. Then the total runtime will be $O(n^2 k d^2 ).$ 

Let $L^Y$ be the kernel of $\mathbb{P}_{L}$ conditioned on the inclusion of items in $Y.$
The eigenvalues of $L^Y$ are exactly the roots of $g(\lambda).$ By \cref{eq:condition kernel}, $L^Y $ can be rewritten as product of two matrices of rank $\leq d$, thus the nonzero eigenvalues of $L^Y$ can be computed in $O(d^3)$ time. Indeed, let $D_Y : = B_{Y}^{\intercal} ( B_Y C B_Y^{\intercal})^{-1} B_{Y}$ then $L^Y = B_{\tilde{Y}} ( C- C D_Y C)B_{\tilde{Y}}^{\intercal} $ and $\rank(D_Y)\leq k$ and $D_Y$ can be computed in $O(k d^2)$ time (see \cref{eq:condition kernel}).
The matrix $F_Y:= \left((C - C D_Y C) B_{\tilde{Y}}^{\intercal}\right)  B_{\tilde{Y}}  $ has the same characteristic polynomial and nonzero eigenvalues as $L^Y.$ Clearly, $\rank(F_Y ) \leq \rank(B) \leq d$, so $F_Y$ and its eigenvalues can be computed in $O(nd^2)$ time.
\end{proof}

\section{Composable Core-Sets via Local Search}
Here we prove that local search yields composable core-sets for distributions that satisfy a strong form of exchange.
\begin{definition}[$\beta$-strong approximate basis exchange] \label{def:strong basis exchange}
For $\beta \geq 1$, we say $\mu: \binom{[n]}{k} \to \R_{\geq 0}$ satisfy $\beta$-strong approximate basis exchange if, for $ S\in \binom{[n]}{k}$ and $j\not \in S$,
\begin{equation}\label{ineq:rsExchange2}
		    \mu(S)\mu(T)\leq \beta \mu(S-i+j)\mu(T+i-j)
\end{equation}


\end{definition}

\begin{lemma}\label{lem:core-set}
Suppose $\mu:\binom{[n]}{k} \to \R_{\geq 0}$ satisfies $\beta$-strong approximate basis exchange, 
then the Local Search algorithm achieves an $O(\beta)^{k}$-composable core-set 
of size $k$ for the MAP-inference problem for $\mu.$ 
\end{lemma}
\begin{proof}
Consider a partition $P_1\cup \dots \cup P_r$ of $[n]$, and let $C_i \in \binom{P_i}{k}$ be a $\alpha$-local optimum in $ P_i$ with $\mu(C_i) > 0.$ We want to show
\[(\frac{\beta}{\alpha})^k \OPT(C) \geq  \OPT(\bigcup_{i=1}^r P_i) \]
where $C := \bigcup_{i=1}^r C_i.$

Let $S^*$ be such that  $\mu(S^*)= \OPT(\bigcup_{i=1}^r P_i).$ We need the following fact.
\begin{claim}\label{claim:ls core-set}
For any $W \in \binom{[n]}{k}$ with non-empty $(W\setminus C_i),$ there exists $W' \in \binom{[n]}{k}$ s.t. $\abs{W' \setminus C_i} = \abs{W \setminus C_i} -1$ and $ \beta \mu(W')\geq\mu(w). $ 
\end{claim}
	
\begin{proof}[Proof of \cref{claim:ls core-set}]
Take an arbitrary $j \in (W \cap P_i) \setminus C_i.$ There exists $e \in C_i \setminus W$ s.t.
\begin{align*}
    \mu(C_i) \mu(W)  \leq \beta \mu(C_i - e +j) \mu(W +e -j) \leq \frac{\beta}{\alpha} \mu(C_i) \mu(W+e-j)
\end{align*}
Setting $W' =  W+e-j$ and dividing both sides by $\mu(C_i)> 0$ gives the desired inequality, since $\abs{W'\setminus C_i} = \abs{W\setminus C_i} - 1.$ 
\end{proof}
We can iteratively apply \cref{claim:ls core-set} for up to $k$ times to obtain the desired inequality. Indeed, let $W_0 : = S^*$, and for $j\geq 1$ let $i_j\in [r]$ and $W_j\in \binom{[n]}{k}$ be such that $\mu(W_{j-1}) \leq\frac{\beta}{\alpha}\cdot \mu(W_j) $ and $\abs{W_j\setminus C_{i_j} } = \abs{W_{j-1}\setminus C_{i_j}} - 1.$ \cref{claim:ls core-set}  guarantees the existence of such $i_j$ and $W_j,$ as long as $W_{j-1} \not\subseteq C.$ Let $s$ be the minimum index such that $W_s \subseteq C.$ Note that $s \leq k$ and $\mu(W_s) \leq \OPT(C).$ We have
	\begin{align*}
	  \OPT([n]) = \mu(W_0) \leq \frac{\beta}{\alpha}\cdot \mu(W_1)  &\leq (\frac{\beta}{\alpha})^2\cdot \mu(W_2) \leq \cdots\\ \leq
	  &(\frac{\beta}{\alpha})^s \mu(W_s) \leq (\frac{\beta}{\alpha})^k \OPT(C)   
	\end{align*}
\end{proof}
\cref{thm:core-set} is a direct consequence of \cref{lem:core-set} and the fact that strongly Rayleigh (log concave resp.) distributions satisfy $k^{O(k)}$-strong approximate basis exchange ($2^{O(k^2)}$-strong approximate basis exchange resp.) \cite{ALOVV21}.

    \PrintBibliography
\end{document}